\def\year{2019}\relax
\documentclass[letterpaper]{article} 
\usepackage{aaai19}  
\usepackage{times}  
\usepackage{helvet}  
\usepackage{courier}  
\usepackage{url}  
\usepackage{graphicx}  

\usepackage[utf8]{inputenc} 
\usepackage[T1]{fontenc}    
\usepackage{booktabs}       
\usepackage{amsfonts}       
\usepackage{nicefrac}       
\usepackage{microtype}      
\usepackage{subcaption}
\usepackage[colorinlistoftodos]{todonotes}
\usepackage{tabularx}
\usepackage{array}
\usepackage[para]{footmisc}
\usepackage{amssymb,amsthm,amsmath}
\usepackage{cleveref}
\DeclareMathOperator{\trace}{trace}
\DeclareMathOperator{\rank}{rank}

\newtheorem{theorem}{Theorem}

\newtheorem{proposition}{Proposition}
\newtheorem{corollary}{Corollary}
\newtheorem{remark}{Remark}

\title{Convex Formulations for Fair Principal Component Analysis}
\author{Matt Olfat \; Anil Aswani\\
	UC Berkeley\\
	Berkeley, CA 94720\\
}

\begin{document}
	\maketitle

	\begin{abstract}
	Though there is a growing literature on fairness for supervised learning, incorporating fairness into unsupervised learning has been less well-studied.  This paper studies fairness in the context of principal component analysis (PCA).  We first define fairness for dimensionality reduction, and our definition can be interpreted as saying a reduction is fair if information about a protected class (e.g., race or gender) cannot be inferred from the dimensionality-reduced data points.  Next, we develop convex optimization formulations that can improve the fairness (with respect to our definition) of PCA and kernel PCA.  These formulations are semidefinite programs, and we demonstrate their effectiveness using several datasets.  We conclude by showing how our approach can be used to perform a fair (with respect to age) clustering of health data that may be used to set health insurance rates.
	
	\end{abstract}
	
	\section{Introduction}
	\label{sec:intro}
	
	
	Despite the success of machine learning in informing policies and automating decision-making, there is growing concern about the fairness (with respect to protected classes like race or gender) of the resulting policies and decisions \cite{miller2015can,rudin2013predictive,angwin2016machine,munoz2016big}.  Hence, several groups have studied how to define fairness for supervised learning \cite{hardt2016equality,calders2009building,dwork2012fairness,zliobaite2015relation} and developed supervised learners that maintain high prediction accuracy while reducing unfairness \cite{berk2017fairness,chouldechova2017fair,hardt2016equality,zafar2017,olfat2017spectral}.
	
	
	However, fairness in the context of unsupervised learning has not been well-studied to date.  One reason is that fairness is easier to define in the supervised setting, where positive predictions can often be mapped to positive decisions (e.g., an individual who is predicted to not default on a loan maps to the individual being offered a loan).  Such notions of fairness cannot be used for unsupervised learning, which does not involve making predictions. A second reason is that it is not obvious why fairness is an issue of relevance to unsupervised learning, since predictions are not made.
	
	\subsection{Relevance of fairness to unsupervised learning}
	
	Fairness is important to unsupervised learning: first, unsupervised learning is often used to generate qualitative insights from data.  Examples include visualizing high-dimensional data through dimensionality-reduction and clustering data to identify common trends or behaviors.  If such qualitative insights are used to generate policies, then there is an opportunity to introduce unfairness in the resulting policies if the results of the unsupervised learning are unequal for different protected classes (e.g., race or gender).  We present such an example in Section \ref{sec:experiment} using individual health data.
	
	Second, unsupervised learning is often used as a preprocessing step for other learning methods.  For instance, dimensionality reduction is sometimes performed prior to clustering, and hence fair dimensionality reduction could indirectly provide methods for fair clustering.  Similarly, there are no fairness-enhancing versions of most supervised learners.  Consequently, techniques for fair unsupervised learning could  be combined with state-of-the-art supervised learners to develop new fair supervised learners.  In fact, the past work most related to this paper concerns techniques that have been developed to generate fair data transformations that maintaining high prediction accuracy for classifiers that make predictions using the transformed data \protect\cite{dwork2012fairness,zemel2013learning,feldman2015certifying}; however, these past works are most accurately classified as supervised learning because the data transformations are computed with respect to a label used for predictions.
	
	We briefly review this work.  \citeauthor{dwork2012fairness} \shortcite{dwork2012fairness} propose a linear program that maps individuals to probability distributions over possible classifications such that similar individuals are classified similarly. \citeauthor{zemel2013learning} \shortcite{zemel2013learning} and \citeauthor{calmon2017optimized} \shortcite{calmon2017optimized} generate an intermediate representation for fair clustering using a non-convex formulation that is difficult to solve. \citeauthor{feldman2015certifying} \shortcite{feldman2015certifying} propose an algorithm that scales data points such that the distributions of features, conditioned on the protected attribute, are matched; however, this approach makes the restrictive assumption that predictions are monotonic with respect to each dimension. \citeauthor{chierichetti2017fair} \shortcite{chierichetti2017fair} directly perform fair clustering by approximating an NP-hard preprocessing step; however, this approach only applies to specific clustering techniques whereas the approach we develop can be used with arbitrary clustering techniques.  Finally, a series of work has emerged using auto-encoders in the the context of deep classification. This area is promising, but suffers from a lack of theoretical guarantees and is further oriented almost entirely around an explicit classification task \cite{beutel2017data,zhang2018mitigating}. In contrast, our method has applications in both supervised and unsupervised learning tasks, and well-defined convergence and optimality guarantees.
	
	\subsection{Outline and novel contributions}
	
	This paper studies fairness for principal component analysis (PCA), and we make three main contributions: First, in Section \ref{sec:notion} we propose and motivate a novel quantitative definition of fairness for dimensionality reduction. Second, in Section \ref{sec:formulation} we develop convex optimization formulations for fair PCA and fair kernel PCA. Third, in Section \ref{sec:experiment} we demonstrate the efficacy of our semidefinite programming (SDP) formulations using several datasets, including using fair PCA as preprocessing to perform fair (with respect to age) clustering of health data that can impact health insurance rates.
	
	\section{Notation}
	\label{sec:notation}

	Let $[n] = \{1,\ldots,n\}$, $\mathbf{1}(u)$ be the Heaviside function, and let $\mathbf{e}$ be the vector whose entries are all 1.  A positive semidefinite matrix $U$ with dimensions $q\times q$ is denoted $U\in\mathbb{S}^q_+$ (or $U \succeq 0$ when dimensions are clear).  We use the notation $\langle\cdot,\cdot\rangle$ to denote the inner product and $\mathbb{I}$ the identity matrix.
	
	Our data consists of 2-tuples $(x_i,z_i)$ for $i=1,\ldots,n$, where $x_i\in\mathbb{R}^p$ are a set of features, and $z_i\in\{-1,1\}$ label a protected class. For a matrix $W$, the $i$-th row of $W$ is denoted $W_i$. Let $X\in\mathbb{R}^{n\times p}$ and $Z\in\mathbb{R}^n$ be the matrices so that $X_i = (x_i - \overline{x})^\textsf{T}$ and $Z_i = z_i$, where $\overline{x} = \frac{1}{n}\sum_i x_i$.  Also, we use the notation $\Pi : \mathbb{R}^p\rightarrow\mathbb{R}^d$ to refer to a function that performs dimensionality reduction on the $x_i$ data, where $d$ is the dimension of the dimensionality-reduced data.
	
	Let $P = \{i : z_i = +1\}$ be the set of indices where the protected class is positive, and similarly let $N = \{i : z_i=-1\}$ be the set of indices where the protected class is negative.
	We use $\#P$ and $\#N$ for the cardinality of these sets. Furthermore, we define $X_+$ to be the matrix whose rows are $x_i^\textsf{T}$ for $i\in P$, and we similarly define $X_-$ to be the matrix whose rows are $x_i^\textsf{T}$ for $i\in N$. Next, let $\widehat{\Sigma}_+$ and $\widehat{\Sigma}_-$ be the sample covariances matrices of $X_+$ and $X_-$, respectively.
	
	For a kernel function $k:\mathbb{R}^p\times\mathbb{R}^p\rightarrow\mathbb{R}_+$, let $K(X,X') = [ k(X_i^{\vphantom{'}},X_j')]_{ij}$
	be the transformed Gram matrix.  Since the \emph{kernel trick} involves replacing $x_i^\textsf{T}x_j$ with $K(x_i,x_j)$, the benefit of the above notation is it allows us to replace $X(X')^\textsf{T}$ with $K(X,X')$ as part of applying the kernel trick.
	
	\section{Fairness for dimensionality reduction}
	\label{sec:notion}

	Definitions of fairness for supervised learning \cite{hardt2016equality,dwork2012fairness,calders2009building,zliobaite2015relation,feldman2015certifying,chouldechova2017fair,berk2017fairness} specify that predictions conditioned on the protected class are roughly equivalent.  However, these fairness notions cannot be used for dimensionality reduction because predictions are not made in unsupervised learning.  This section discusses fairness for dimensionality reduction.  We first provide and motivate a general quantitative definition of fairness, and then present several important cases of this definition.  
	
	\subsection{General definition}
	
	Consider a fixed classifier $h(u,t) : \mathbb{R}^d\times\mathbb{R}\rightarrow \{-1,+1\}$ that inputs features $u\in\mathbb{R}^d$ and a threshold $t$, and predicts the protected class $z\in\{-1,+1\}$.  We say that a dimensionality reduction $\Pi : \mathbb{R}^p\rightarrow\mathbb{R}^d$ is $\Delta(h)$-fair if
	\begin{equation}
	\label{eq:fairness}
	\begin{aligned}
	&\Big|\mathbb{P}\big[h(\Pi(x), t) = +1 \big| z = +1\big]\\
	&-\mathbb{P}\big[h(\Pi(x), t) = +1 \big| z = -1\big]\Big| \leq \Delta(h), \ \forall t\in\mathbb{R}.
	\end{aligned}
	\end{equation}
	Moreover, let $\mathcal{F}$ be a family of classifiers.  Then we say that a dimensionality reduction $\Pi : \mathbb{R}^p\rightarrow\mathbb{R}^d$ is $\Delta(\mathcal{F})$-fair if it is $\Delta(h)$-fair for all classifiers $h\in\mathcal{F}$.

	Our fairness definition can be interpreted via classification: Observe that the first term in the left-hand-side of (\ref{eq:fairness}) is the true positive rate of the classifier $h$ in predicting the protected class using the dimensionality-reduced variable $\Pi(x)$ at threshold $t$, and the second term is the corresponding false positive rate.  Thus, $\Delta(h)$ in our definition (\ref{eq:fairness}) can be interpreted as bounding the accuracy of the classifier $h$ in predicting the protected class using the dimensionality-reduced variable $\Pi(x)$. 
	
	Note that \cref{eq:fairness} is analogous to \textit{disparate impact} for classifiers \cite{calders2009building,feldman2015certifying}, where we require that treatment not vary at all between protected classes. This has often been criticized as too strict of a notion in classification, and so alternate notions of fairness have been developed, such as \textit{equalized odds} and \textit{equalized opportunity} \cite{hardt2016equality}. Instead of equalizing all treatment across protected classes, these notions instead focus on equalizing error rates; for example, in the case of lending, equalized odds would require nondiscrimination \textit{among all applicants of similar FICO scores}, whereas disparate impact would require nondiscrimination among all applicants. This may be preferred in cases where $y$ and $z$ are strongly correlated. In any case, it can easily be incorporated into our model by simply further conditioning the two terms on the left-hand-side of \cref{eq:fairness} on the main label, $y$.

	\subsection{Motivation}

	The above is a meaningful definition of fairness for dimensionality reduction because it implies that a supervised learner using fair dimensionality-reduced data will itself be fair. This is formalized below:
	
	\begin{proposition}\label{prop:fairness}
	Suppose we have a family of classifiers $\mathcal{F}$ and a dimensionality reduction $\Pi$ that is $\Delta(\mathcal{F})$-fair.  Then any classifier that is selected from $\mathcal{F}$ to predict a label $y\in\{-1,+1\}$ using $\Pi(x)$ as features will have disparate impact less than $\Delta(\mathcal{F})$.
	\end{proposition}
	
	\Cref{prop:fairness} follows directly from our definition of fairness.  We anticipate that in most situations the goal of the dimensionality reduction would not be to explicitly predict the protected class. Thus, our approach of bounding intentional discrimination on $z$ represents a conservative bound on any discrimination that may incidentally arise when performing classificiation using the family $\mathcal{F}$ or when deriving qualitative insights form the results of unsupervised learning.
	
	

	\subsection{Special cases}

	An important special case of our definition occurs for the family $\mathcal{F}_c = \{h(u,t) = \mathbf{1}(u \leq w + t) : w \in \mathbb{R}^d\}$, where the inequality in this expression should be interpreted element-wise.  In this case, our definition can be rewritten as $\textstyle\sup_u\big|F_{\Pi(x)|z=+1}(u) - F_{\Pi(x)|z = -1}(u)\big| \leq \Delta(\mathcal{F}_c)$, where $F$ is the cumulative distribution function (c.d.f.) of the random variable in the subscript.  Restated, for this family our definition is equivalent to saying $\Delta(\mathcal{F})$ is a bound on the Kolmogorov distance between $\Pi(x)$ conditioned on $z=\pm 1$ (i.e., the left-hand side of the above equation).
	
	
	Other important cases are the family of linear support vector machines (SVM's) $\mathcal{F}_v = \{h(u,t) = \mathbf{1}(w^\textsf{T}u - t \leq 0) : w\in\mathbb{R}^d\}$ and the family of kernel SVM's $\mathcal{F}_k$ for a fixed kernel $k$.  These important cases are used in Section \ref{sec:formulation} to propose formulations for fair PCA and fair kernel PCA.
	
	Next, we briefly discuss empirical estimation of $\Delta(\mathcal{F})$.  An empirical estimate of $\Delta(h)$ is given by $\widehat{\Delta}(h) = \sup_t|\frac{1}{\#P}\sum_{i\in P}\mathbf{1}(h(\Pi(x), t) = +1) - \frac{1}{\#N}\sum_{i\in N}\mathbf{1}(h(\Pi(x), t) = +1)|$. Similarly, we define $\widehat{\Delta}(\mathcal{F}) = \sup\{\widehat{\Delta}(h)\ |\ h\in\mathcal{F}\}$.  Last, note that we can provide high probability bounds of the actual fairness level in terms of these empirical estimates:
	
	\begin{proposition}
		\label{pro:glivenko}
		Consider a fixed family of classifiers $\mathcal{F}$.  If the samples $(x_i,z_i)$ are i.i.d., then for any $\delta > 0$ we have with probability at least $1-\exp(-n\delta^2/2)$ that $\Delta(\mathcal{F})\leq\widehat{\Delta}(\mathcal{F})+8\sqrt{\mathcal{V}(\mathcal{F})/n}+\delta$, where $\mathcal{V}(\mathcal{F})$ is the VC dimension of the family $\mathcal{F}$.
	\end{proposition}

	This result follows from the triangle inequality, bounding $\Delta(\mathcal{F})$ with $\hat{\Delta}(\mathcal{F})$ plus a generalization error, for which there are standard bounds via Dudley's entropy integral \cite{wainwright2017high}.
	
	
	\begin{remark}
		Recall that $\mathcal{V}(\mathcal{F}_c) = d + 1$ \cite{shorack2009empirical}, and that $\mathcal{V}(\mathcal{F}_v) = d+1$ \cite{wainwright2017high}.  This means $\widehat{\Delta}(\mathcal{F}_c)$ and $\widehat{\Delta}(\mathcal{F}_v)$ will be accurate when $n$ is large relative to $d$.
	\end{remark}
	
	\section{Projection defined by PCA}
	
	Our approach to designing an algorithm for fair PCA will begin by first studying the convex relaxation of a non-convex optimization problem whose solution provides the projection defined by PCA.  First, note that computation of the first $d$ PCA components $v_i$ for $i=1,\ldots,d$ can be written as the following non-convex optimization problem: $\max\{\textstyle\sum_{i=1}^{d}v_i^{\textsf{T}}X^{\textsf{T}}Xv_i\ |\ \|v_i\|_2 \leq 1,\ \ v_i^\textsf{T}v_j^{\vphantom{\textsf{T}}} = 0, \text{ for } i\neq j\}$. Now suppose we define the matrix $P = \sum_{i=1}^dv_i^{\vphantom{\textsf{T}}}v_i^\textsf{T}$, and note $\sum_{i=1}^{d}v_i^{\textsf{T}}X^{\textsf{T}}Xv_i = \sum_{i=1}^d\langle X^{\textsf{T}}X, v_i^{\vphantom{\textsf{T}}}v_i^\textsf{T}\rangle = \langle X^{\textsf{T}}X, P\rangle$.  Thus, we can rewrite the above optimization problem as
	\begin{equation}
	\label{eq:topeig2}
	\max\big\{\langle X^{\textsf{T}}X,P\rangle\ \big|\ \rank(P) \leq d,\, \mathbb{I}\succeq P\succeq 0\big\}.
	\end{equation}
	
	In the above problem, we should interpret the optimal $P^*$ to be the projection matrix that projects $x\in\mathbb{R}^p$ onto the $d\,\!$ PCA components (still in the original $p$-dimensional space). Next, we consider a convex relaxation of (\ref{eq:topeig2}).  Since $\mathbb{I}-P \succeq 0$, the usual nuclear norm relaxation is equivalent to the trace \cite{recht2010guaranteed}.  So our convex relaxation is
	\begin{equation}
	\label{eq:topeig3}
	\max\big\{\langle X^{\textsf{T}}X,P\rangle\ \big|\ \trace(P) \leq d,\; \mathbb{I}\succeq P\succeq 0\big\}.
	\end{equation}
	Note that this base model is the same as that used by \cite{arora2013stochastic}. The following result shows that we can recover the first $d$ PCA components from any $P^*$ that solves (\ref{eq:topeig3}).
	
	
	\begin{theorem}
		\label{thm:cr}
		Let $P^*$ be an optimal solution of (\ref{eq:topeig3}), and consider its diagonalization: $P^* = \sum_{i=1}^p\lambda^*_iv_i^{\vphantom{\mathsf{T}}}v_i^\mathsf{T}$, where $v_i$ is an orthonormal basis, and (without loss of generality) the $\lambda^*_i$ are in non-increasing order.  Then the positive semidefinite $P^{**} \triangleq \sum_{i=1}^dv_i^{\vphantom{\mathsf{T}}}v_i^\mathsf{T}$ is an optimal solution to (\ref{eq:topeig2}).
	\end{theorem}
	
	\begin{proof}
		We consider two cases. First, if $\rank(P^*)\le d$ then $\lambda^*_i\in\{0,1\}$ or $v_i^{\textsf{T}}X^{\textsf{T}}X^{\vphantom{\textsf{T}}}v_i^{\vphantom{\textsf{T}}}=0$ for all $i$, since otherwise we could increase $\lambda^*_i$ if $v_i^{\textsf{T}}X^{\textsf{T}}X^{\vphantom{\textsf{T}}}v_i^{\vphantom{\textsf{T}}}>0$ (or vice versa) to improve the objective while maintaining feasibility. It follows that $\langle X^{\textsf{T}}X,P^*\rangle=\langle X^{\textsf{T}}X,P^{**}\rangle$. This means that $P^{**}$ is optimal for (\ref{eq:topeig3}); since it is also feasible for (\ref{eq:topeig2}), we are done. Second, if $\rank(P^*)>d$ then $0<\lambda_d^*<1$ since the $\lambda_i^*$ are ordered. Consider $\tilde{P}\triangleq(P^*-cP^{**})/(1-c),c=\min\{\lambda^*_d,1-\lambda^*_d\}$. Note that $\tilde{P}$ is feasible for (\ref{eq:topeig3}), and that $P^*$ is a strict convex combination of $P^{**}$ and $\tilde{P}$. All points between $\tilde{P}$ and $P^{**}$ are feasible by convexity, and so the optimality of $P^*$ implies that $P^{**}$ and $\tilde{P}$ must also be optimal for (\ref{eq:topeig3}) by linearity of the objective (i.e., at least one must have objective value no less than that of $P^*$, but if one had a strictly better objective value than the other, then no strict convex combination of the two could be optimal). The result then follows from the optimality of $P^{**}$ for (\ref{eq:topeig3}) and feasibility for (\ref{eq:topeig2}).
	\end{proof}

	We conclude this section with two useful results on the spectral norm $\|\cdot\|_2$ of a symmetric matrix.
	
	\begin{figure*}[!t]
		\begin{center}
			\begin{subfigure}[t]{0.24\linewidth}
				\includegraphics[width=\linewidth]{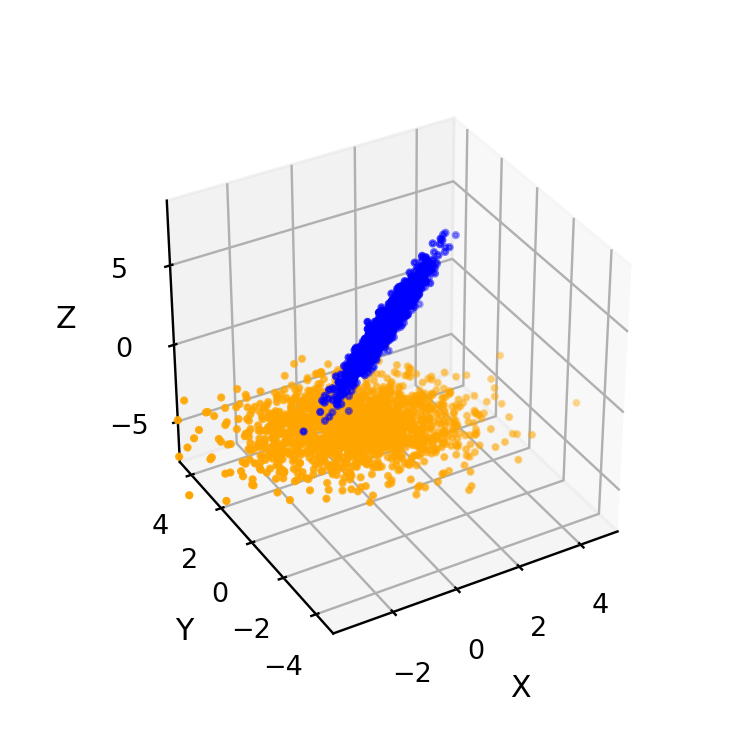}
				\caption{\label{fig:rand3D} Original data}
			\end{subfigure}\hfill
			\begin{subfigure}[t]{0.24\linewidth}
				\includegraphics[width=\linewidth]{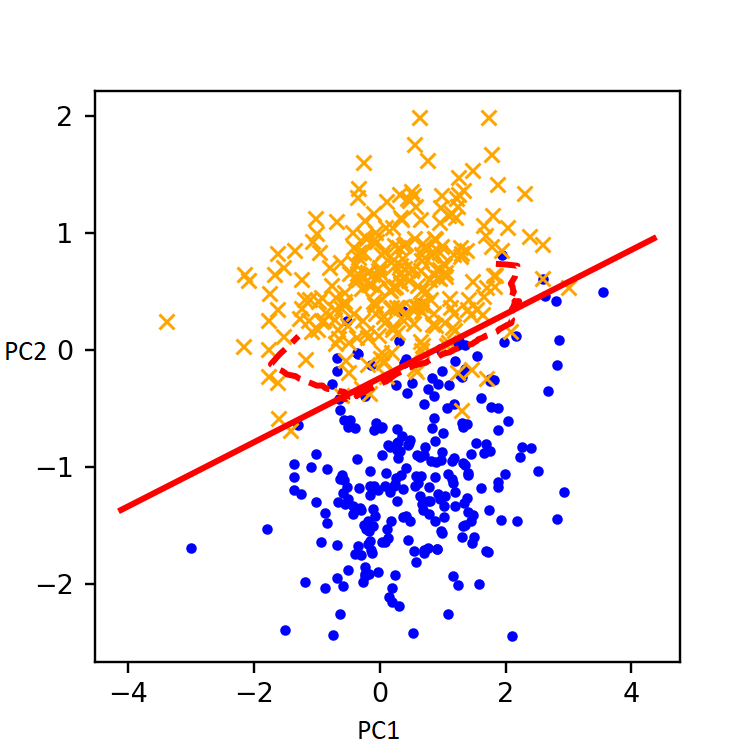}
				\caption{\label{fig:randNone} PCA}
			\end{subfigure}\hfill
			\begin{subfigure}[t]{0.24\linewidth}
				\includegraphics[width=\linewidth]{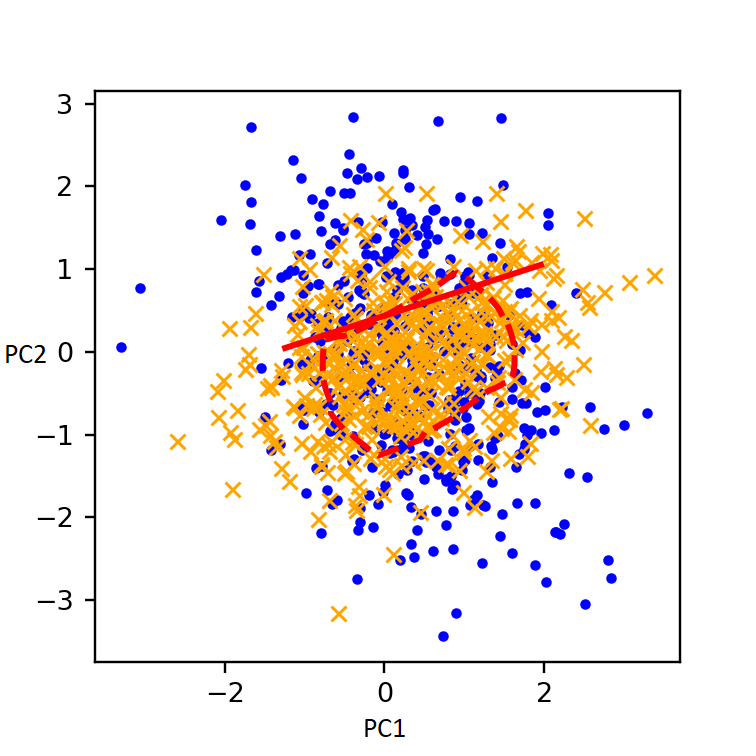}
				\caption{\label{fig:randOne} FPCA - Mean con.}
			\end{subfigure}\hfill
			\begin{subfigure}[t]{0.24\linewidth}
				\includegraphics[width=\linewidth]{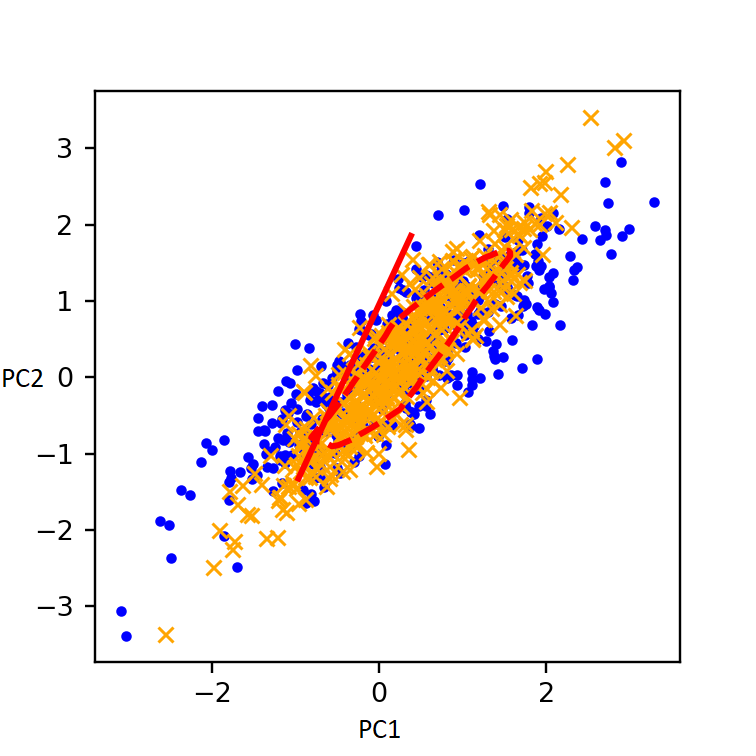}
				\caption{\label{fig:randBoth} FPCA - Both con.}
			\end{subfigure}
			\caption{\label{fig:allrand} Comparison of PCA and FPCA on synthetic data. In each plot, the thick red line is the optimal linear SVM separating by color, and the dotted line is the optimal Gaussian kernel SVM.}
		\end{center}
	\end{figure*}
	
	\begin{theorem}
		\label{thm:sym}
		Let $Q$ be a symmetric matrix, and suppose $\varphi \geq \|Q\|_2$.  Then $\|Q\|_2 = \max\{\|Q + \varphi\mathbb{I}\|_2, \|-Q + \varphi\mathbb{I}\|_2\} - \varphi$.
	\end{theorem}
	
	\begin{proof}
		First diagonalize $Q = \sum_{i=1}^p\lambda_i^{\vphantom{\textsf{T}}}v_i^{\vphantom{\textsf{T}}}v_i^\textsf{T}$, with orthonormal basis $v_i$ and (without loss of generality) $\lambda_i$ in non-increasing order. Then $+Q + \varphi\mathbb{I} = \textstyle\sum_{i=1}^p(+\lambda_i^{\vphantom{\textsf{T}}}+\varphi)v_i^{\vphantom{\textsf{T}}}v_i^\textsf{T},\; -Q + \varphi\mathbb{I} = \textstyle\sum_{i=1}^p(-\lambda_i^{\vphantom{\textsf{T}}}+\varphi)v_i^{\vphantom{\textsf{T}}}v_i^\textsf{T}$.	But by construction $\lambda_i + \varphi \geq 0$ and $-\lambda_i+\varphi \geq 0$ for all $i = 1,\ldots,p$.  Thus $\|Q + \varphi\mathbb{I}\|_2 = \lambda_1 + \varphi$ and $\|-Q+\varphi\mathbb{I}\|_2 = -\lambda_p + \varphi$.  The result follows since $\|Q\|_2 = \max\{\lambda_1, -\lambda_p\}$.
	\end{proof}
	
	\begin{corollary}
		\label{cor:sn}
		Let $Q$ be a symmetric matrix, and suppose $\varphi \geq \|Q\|_2$.  If $V$ is such that $V^\mathsf{T}V=\mathbb{I}$, then $\|V^\mathsf{T}QV\|_2 = \max\{\|V^\mathsf{T}(Q+\varphi\mathbb{I})V\|_2, \|V^\mathsf{T}(-Q+\varphi\mathbb{I})V\|_2\} - \varphi$.
	\end{corollary}
	
	\begin{proof}
		First note that $V^\textsf{T}(Q+\varphi\mathbb{I})V = V^\textsf{T}QV+\varphi\mathbb{I}$ and that $V^\textsf{T}(-Q+\varphi\mathbb{I})V = -V^\textsf{T}QV+\varphi\mathbb{I}$. Since the spectral norm is submultiplicative, this means $\|V^\textsf{T}QV\|_2 \leq \|V^\textsf{T}\|_2\|Q\|_2\|V\|_2 \leq \|Q\|_2$.  So $\varphi \geq \|V^\textsf{T}QV\|_2$, and the result follows by applying Theorem \ref{thm:sym} to $V^\textsf{T}QV$.
	\end{proof}
	
	Recall that using the Schur complement allows representation of $\|VRV^\mathsf{T}\|_2$ as a positive semidefinite matrix constraint when $R$ is positive semidefinite \cite{boyd1994linear}.  So the above corollary is useful because it means we can represent $\|VQV^\mathsf{T}\|_2$ using positive semidefinite matrix constraints since $(Q+\varphi\mathbb{I})$ and $(-Q+\varphi\mathbb{I})$ are positive semidefinite by construction.

	\section{Designing formulations for fair PCA}
	\label{sec:formulation}

	
	
	Consider the linear dimensionality reduction $\Pi(x) = V^\mathsf{T}x$ for $V\in\mathbb{R}^{p\times d}$ such that $V^\textsf{T}V = \mathbb{I}$.  Then for linear classifier $h(u,t) = \mathbf{1}(w^\textsf{T}u - t \leq 0)$, definition (\ref{eq:fairness}) simplifies to $\Delta(h) = \sup_t|\mathbb{P}[w^\textsf{T}V^\mathsf{T}x \leq t | z = +1]-\mathbb{P}[w^\textsf{T}V^\mathsf{T}x \leq t| z = -1]|$. But the right-hand side is the Kolmogorov distance between $w^\textsf{T}V^\mathsf{T}x$ conditioned on $z=\pm 1$, which is upper bounded (as can be seen trivially from its definition) by the total variation distance.  Consequently, applying Pinsker's inequality \cite{massart2007concentration} gives $\Delta(h) \leq \sqrt{\frac{1}{2}\mathcal{KL}\big(w^\textsf{T}V^\mathsf{T}X_-\, \big|\big|\, w^\textsf{T}V^\mathsf{T}X_+\big)},$ where $\mathcal{KL}(\cdot||\cdot)$ is the Kullback-Leibler divergence, $X_+$ is the random variable $[x | z = +1]$, and $X_-$ is the random variable $[x | z=-1]$.  For the special case $X_+\sim\mathcal{N}(\mu_+,\Sigma_+)$ and $X_-\sim\mathcal{N}(\mu_-,\Sigma_-)$, we have \cite{kullback1997information}:
	\begin{equation}
	\label{eqn:klbd}
	\Delta(h) \leq \sqrt{\frac{1}{4}\left(\frac{s_-}{s_+}+\frac{(m_+-m_-)^2}{s_+}+\log\frac{s_+}{s_-}-1\right)}.
	\end{equation}
	where $s_+ = w^\textsf{T}V^\mathsf{T}\Sigma_+Vw$, $s_- = w^\textsf{T}V^\mathsf{T}\Sigma_-Vw$, $m_+ = w^\textsf{T}V^\mathsf{T}\mu_+$, and $m_- = w^\textsf{T}V^\mathsf{T}\mu_-$.  The key observation here is that (\ref{eqn:klbd}) is minimized when $s_+ = s_-$ and $m_+ = m_-$, and we will use this insight to propose constraints for FPCA. If $X_+$ and $X_-$ are not Gaussian, the three-point property may be used to obtain a similar bound with a couple extra terms involving the divergence between $X_+$ and a normal distribution with the same mean and variance (and the analog for $X_-$).
	
	We first design constraints for the non-convex formulation (\ref{eq:topeig2}) so that $\hat{m}_+ - \hat{m}_- = w^\textsf{T}V^\mathsf{T}f$ has small magnitude, where $f = \hat{\mu}_+ - \hat{\mu}_- = \frac{1}{\#P}\sum_{i\in P}x_i - \frac{1}{\#N}\sum_{i\in N}x_i$.  Note we make the identification $P = VV^\mathsf{T}$ because of the properties of $P$ in (\ref{eq:topeig2}) and since $V^\textsf{T}V = \mathbb{I}$.  Observe that $w^\textsf{T}V^\mathsf{T}f$ is small if $V^\mathsf{T}f$ is small, which can be formulated as
	\begin{equation}
	\label{eqn:m1}
	\|V^\mathsf{T}f\|^2 = \langle VV^\textsf{T}, ff^\textsf{T}\rangle = \langle P, ff^\textsf{T}\rangle\leq \delta^2,
	\end{equation}
	where $\|\cdot\|$ is the $\ell_2$-norm, and $\delta$ is a bound on the norm.  This (\ref{eqn:m1}) is a linear constraint on $P$.
	
	We next design constraints for the non-convex formulation (\ref{eq:topeig2}) so that $\hat{s}_+ - \hat{s}_- = w^\textsf{T}V^\mathsf{T}(\widehat{\Sigma}_+-\widehat{\Sigma}_-)Vw$ has small magnitude.  Recall the identification $P = VV^\textsf{T}$ because of the properties of $P$ in (\ref{eq:topeig2}) and since $V^\textsf{T}V = \mathbb{I}$.  Next observe that $w^\textsf{T}V^\mathsf{T}(\widehat{\Sigma}_+-\widehat{\Sigma}_-)Vw$ is small if $V^\mathsf{T}(\widehat{\Sigma}_+-\widehat{\Sigma}_-)V$ is small.  Let $Q = \widehat{\Sigma}_+-\widehat{\Sigma}_-$, then using Corollary \ref{cor:sn} gives
	\begin{multline}
	\label{eqn:c}
	\mu+\varphi \geq \|V^\textsf{T}QV\|_2 +\varphi \\
	= \max\{\|V^\mathsf{T}(Q+\varphi\mathbb{I})V\|_2, \|V^\mathsf{T}(-Q+\varphi\mathbb{I})V\|_2\} \\
	=\max\{\|VV^\textsf{T}(Q+\varphi\mathbb{I})VV^\mathsf{T}\|_2, \|VV^\textsf{T}(-Q+\varphi\mathbb{I})VV^\mathsf{T}\|_2\} \\
	=\max\{\|P(Q+\varphi\mathbb{I})P\|_2, \|P(-Q+\varphi\mathbb{I})P\|_2\},
	\end{multline}
	where $\varphi \geq \|\widehat{\Sigma}_+-\widehat{\Sigma}_-\|_2$, and $\mu$ is a bound on the norm. Note (\ref{eqn:c}) can be rewritten as SDP constraints using a standard reformulation for the spectral norm \cite{boyd1994linear}.
	
	
	We design an SDP formulation for FPCA by combining the above elements.  Though (\ref{eq:topeig2}) with constraint  (\ref{eqn:m1}) and (\ref{eqn:c}) is a non-convex problem for FPCA, we showed in Theorem \ref{thm:cr} that (\ref{eq:topeig3}) was an exact relaxation of (\ref{eq:topeig2}) after extracting the $d$ largest eigenvectors of the solution of (\ref{eq:topeig3}).  Thus, we propose the following SDP formulation for FPCA:
	\begin{subequations}\label{eqn:sdpf}
		\begin{align}
		\max\ &\langle X^{\textsf{T}}X,P\rangle - \mu t\\
		\text{s.t. } & \!\trace(P) \leq d,\ \mathbb{I}\succeq P\succeq 0\\
		&\langle P, ff^\textsf{T}\rangle\leq \delta^2\label{eqn:mc}\\
		&\!\!\begin{bmatrix}t\mathbb{I} & PM_+ \\ M_+^{\textsf{T}}P & \mathbb{I}\end{bmatrix}\succeq 0,\label{eqn:cc1}\\
		&\!\!\begin{bmatrix}t\mathbb{I} & PM_- \\ M_-^{\textsf{T}}P & \mathbb{I}\end{bmatrix}\succeq 0\label{eqn:cc2}
		\end{align}
	\end{subequations}
	
	where $M_iM_i^{\textsf{T}}$ is the Cholesky decomposition of $iQ+\varphi \mathbb{I}$ ($i\in\{-,+\}$), $\varphi \geq \|\widehat{\Sigma}_+-\widehat{\Sigma}_-\|_2$, (\ref{eqn:mc}) is called the \emph{mean constraint} and denotes the use (\ref{eqn:m1}), and (\ref{eqn:cc1}) and (\ref{eqn:cc2}) are called the \emph{covariance constraints} and are the SDP reformulation of (\ref{eqn:c}).  Our convex formulation for FPCA consists of solving (\ref{eqn:sdpf}) and then extracting the $d$ largest eigenvectors from the optimal $P^*$.  
	
	We can apply the kernel trick to (\ref{eqn:sdpf}) to develop an SDP for F-KPCA. For brevity, we only note the differences with (\ref{eqn:sdpf}): $Q$ would be replaced with $Q_k = K(X,X_+)K(X_+,X)-K(X,X_-)K(X_-,X)$ and $f$ with $f_k = \frac{1}{\#P}K(X,X_+)\mathbf{e} - \frac{1}{\#N}K(X,X_-)\mathbf{e}$. $M_+$ and $M_-$ would then be the Cholesky decompositions of the analogous matrices, and $\varphi$ would also have to be set no less than $\|Q_k\|_2$. K-FPCA is then the top $d$ eigenvectors of the optimal solution of the resulting SDP.
	
	Furthermore, this method may be extended to multiple protected attributes by replicating constraints (\ref{eqn:mc}), (\ref{eqn:cc1}) \& (\ref{eqn:cc2}) appropriately. That is, for secondary protected attribute $z'$, we may define the appropriate $f$, $M_+$ and $M_-$ values and add the analogous constraints. Note that this will only abet ``pairwise fairness", or fairness with respect to each of the protected attributes individually. To attain ``joint fairness", or fairness with respect to both terms simultaneously, we would need to recreate constraints (\ref{eqn:mc}), (\ref{eqn:cc1}) \& (\ref{eqn:cc2}) for $z'$ as well as the interaction between $z$ and $z'$, which we can denote by $z_{\textrm{inter}}=\left[\frac{(z_i+1)(z'_i+1)}{2}\right]_i$. This is important because it is possible to attain mathematical fairness with respect to gender and race, for example, while still exhibiting discrimination towards women of one particular racial group.
	
	

	\section{Experimental results}
	\label{sec:experiment}
	
	We use synthetic and real datasets from the UC Irvine Machine Learning Repository \cite{Lichman:2013} to demonstrate the efficacy of our SDP formulations. We also show how FPCA can be used to minimize discrimination in health insurance rates (with respect to age). For any SVM run, tuning parameters were chosen using 5-fold cross-validation, and data was normalized to have unit variance in each field. Due to constraints on space, further experimental results and a overview of and comparison to the method of \citeauthor{calmon2017optimized} are presented in the appendix.

	\subsection{Synthetic Data}
	\label{sec:simulateddata}

	We sampled 1000 points each from $X_+$ and $X_-$ distributed as different 3-dimensional multivariate Gaussians, and these points are shown in \Cref{fig:rand3D}. \Cref{fig:randNone} displays the results of dimensionality reduction using the top two unconstrained principal components of $X$: the resulting separators for linear and Gaussian kernel SVM's are also shown. It is clear that the two sub-populations are readily distinguishable in the lower-dimensional space. \Cref{fig:randOne} displays the analogous information after FPCA with only the mean constraint, and \Cref{fig:randBoth} after FPCA with both constraints. \Cref{fig:randOne,fig:randBoth} clearly display better mixing of the data, and the SVM's conducted afterwards are unable to separate the sub-groups as cleanly as they can in \Cref{fig:randNone}; furthermore, the addition of the covariance constraints (\ref{eqn:cc1}) incentivizes the choosing of a dimensionality reduction that better matches the skew of the entire data set.


	\begin{figure}
	\begin{center}
		\centerline{\includegraphics[width=\columnwidth]{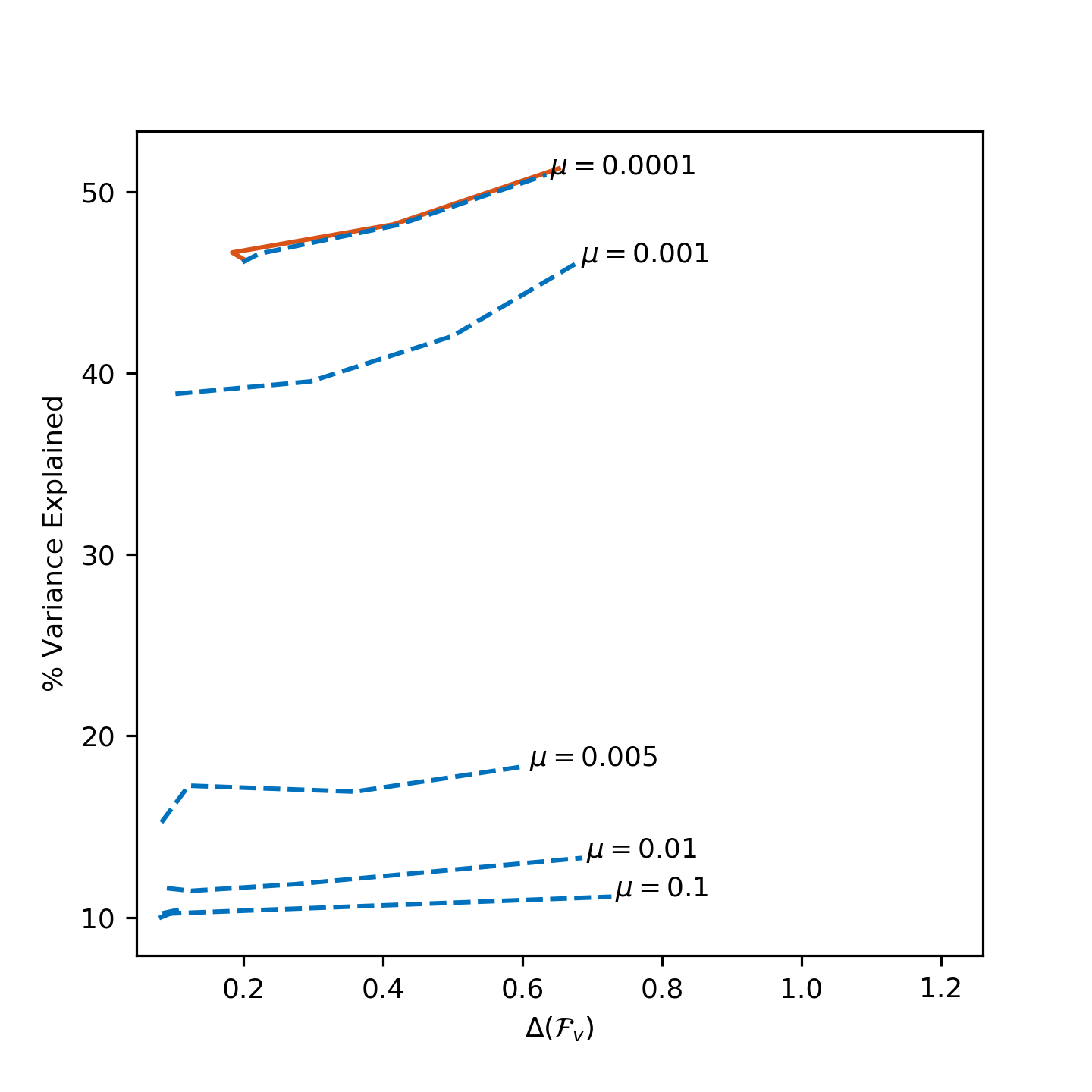}}
		\caption{The sensitivity of FPCA to the $\delta$ and $\mu$ for the wine quality dataset. The full red line represents FPCA with only the mean constraint, and the dotted blue lines denote FPCA with both constraints. For each curve, $\delta\in\{0,0.1,0.3,0.5\}$ was considered.}
		\label{fig:sensitivity}
	\end{center}
	\end{figure}

	\subsection{Real data}
	\label{sec:realdata}
	
	\begin{table*}[!t]
		\caption{$\Delta$-fairness for both linear and Gaussian kernel SVM for PCA and FPCA. Best results for each fairness metric are bolded.}
		\label{tab:results}
		\begin{center}
			\begin{small}
				\begin{sc}
					\begin{tabular}{l|ccc|ccc|cccr}
						\toprule
						& \multicolumn{3}{c}{Unconstrained} & \multicolumn{3}{c}{FPCA - Mean Con.} & \multicolumn{3}{c}{FPCA - Both Con.} & \\
						Data Set& \%var & Lin. & Gaus. & \%var & Lin. & Gaus. & \%var & Lin. & Gaus. \\
						\midrule
						Adult Income & 11.41 & 0.54 & 0.54 & 9.27 & 0.14 & 0.35 & 5.33 & \bf 0.07 & \bf 0.15 \\
						Biodeg \footnotemark{} & 31.16 & 0.2 & 0.35 & 30.46 & 0.14 & 0.29 & 21.45 & \bf 0.10 & \bf 0.28 \\
						E. Coli \footnotemark{} & 65.01 & 0.65 & 0.80 & 54.31 & 0.46 & 0.59 & 53.75 & \bf 0.24 & \bf 0.54 \\
						Energy \footnotemark{} & 84.08 & 0.10 & 0.20 & 66.48 & \bf 0.07 & 0.20 & 62.11 & \bf 0.07 & \bf 0.16 \\
						German Credit & 11.19 & 0.21 & 0.31 & 10.91 & 0.14 & 0.33 & 8.84 & \bf 0.11 & \bf 0.29 \\
						Image & 62.68 & 0.18 & 0.32 & 52.78 & \bf 0.14 & 0.33 & 48.55 & 0.15 & \bf 0.28 \\
						Letter & 42.33 & 0.58 & 0.58 & 29.29 & \bf 0.07 & 0.22 & 23.76 & \bf 0.07 & \bf 0.19 & \\
						Magic \footnotemark{} & 61.91 & 0.32 & 0.33 & 29.57 & \bf 0.11 & \bf 0.21 & 25.36 & 0.12 & 0.30 \\
						Pima \footnotemark{} & 49.00 & 0.30 & 0.37 & 43.98 & \bf 0.17 & 0.26 & 43.26 & 0.18 & \bf 0.25 \\
						Recidivism \footnotemark{} & 56.28 & 0.24 & 0.26 & 46.58 & \bf 0.08 & \bf 0.16 & 39.34 & \bf 0.08 & 0.21 \\
						Skillcraft \footnotemark{} & 40.62 & 0.15 & 0.19 & 29.95 & \bf 0.07 & \bf 0.14 & 25.48 & \bf 0.07 & 0.17 \\
						Statlog & 87.80 & 0.79 & 0.79 & 21.77 & 0.23 & 0.69 & 7.76 & \bf 0.13 & \bf 0.44 \\
						Steel & 46.05 & 0.64 & 0.71 & 40.79 & 0.19 & 0.51 & 11.86 & \bf 0.09 & \bf 0.22 \\
						Taiw. Credit \footnotemark{} & 45.52 & 0.11 & 0.17 & 30.07 & 0.08 & 0.16 & 20.08 & \bf 0.06 & \bf 0.14 \\
						Wine Quality \footnotemark{} & 50.21 & 0.97 & 0.96 & 37.34 & 0.21 & 0.51 & 10.12 & \bf 0.06 & \bf 0.13 \\
						\bottomrule
					\end{tabular}
				\end{sc}
			\end{small}
		\end{center}
	\end{table*}
	
	We next consider a selection of datasets from UC Irvine's online Machine Learning Repository \cite{Lichman:2013}. For each of the datasets, one attribute was selected as a protected class, and the remaining attributes were considered part of the feature space. After splitting each dataset into separate training (70\%) and testing (30\%) sets, the top five principal components were then found for the training sets of each of these datasets three times: once unconstrained, once with (\ref{eqn:sdpf}) with only the mean constraints (and excluding the covariance constraints) with $\delta=0$, and once with (\ref{eqn:sdpf}) with both the mean and covariance constraints with $\delta=0$ and $\mu=0.01$; the test data was then projected onto these vectors. All data was normalized to have unit variance in each feature, which is common practice for datasets with features of incomparable units. For each instance, we estimated $\Delta(\mathcal{F})$ using the test set and for the families of linear SVM's $\mathcal{F}_v$ and Gaussian kernel SVM's $\mathcal{F}_k$. Finally, for each set of principal components $V$, the proportion of variance explained by the components was calculated as $\trace(V\widehat{\Sigma} V^{\textsf{T}}))/\trace(\widehat{\Sigma})$, where $\widehat{\Sigma}$ is the centered sample covariance matrix of training set $X$. \Cref{tab:results} displays all of these results averaged over 5 different training-testing splits.

	We may observe that our additional constraints are largely helpful in ensuring fairness by all definitions. Furthermore, in many cases, this increase in fairness comes at minimal loss in the explanatory power of the principal components. There are a few datasets for which (\ref{eqn:cc1}) appear superfluous. In general, gains in fairness are stronger with respect to $\mathcal{F}_v$; this is to be expected, as $\mathcal{F}_k$ is a highly sophisticated set, and thus more robust to linear projections. Kernel FPCA may be a better approach to tackling this issue, but we leave this for future work. Additional experiments and a comparison to the method of \citeauthor{calmon2017optimized} are shown in the appendix. We find that our method leads to more fairness on almost all datasets.
	
	\subsection{Hyperparameter sensitivity}
	\label{sec:sensitivity}
	
	Next, we consider the sensitivity of our results to hyperparameters $\delta,\mu$, for the Wine Quality dataset. The data was split into training (70\%) and testing (30\%) sets, and the top three fair principle components were found using (\ref{eqn:sdpf}) with only the mean constraint for each candidate $\delta$ and using (\ref{eqn:sdpf}) with both constraints for all combinations of candidate $\delta$ and $\mu$. All data was normalized to have unit variance in each independent feature. We calculated the percentage of the variance explained by the resulting principle components, and we estimated the fairness level $\Delta(\mathcal{F}_v)$ for the family of linear SVM's. This process was run 10 times for random data splits, and the averaged results are plotted in \Cref{fig:sensitivity}. Here, the solid red line represents (\ref{eqn:sdpf}) with only the mean constraint.  On the other hand, the dotted blue lines represent the (\ref{eqn:sdpf}) with both constraints, for the indicated $\mu$.
	
	Adding the covariance constraints and further tightening $\mu$ generally improves fairness and decreases the proportion of variance explained. However, observe that the relative sensitivity of fairness to $\delta$ is higher than that of the variance explained, at least for this dataset. Similarly, increasing $\mu$ decreases the portion of variance explained while resulting in a less discriminatory dataset after the dimensionality reduction. We note that increasing $\mu$ past a certain point does not provide much benefit, and so smaller values of $\mu$ are to be preferred. We found that increasing $\mu$ past 0.1 did not substantively change results further, so the largest $\mu$ that we consider is 0.1. In general, hyperparameters may be set with cross-validation, although (\ref{eqn:klbd}) may serve as guidance. 
	
	
	\subsection{Fair clustering of health data}
	\label{sec:nhanes}
		
	\footnotetext[1]{\cite{mansouri2013quantitative}}
	\footnotetext[2]{\cite{horton1996probabilistic}}
	\footnotetext[3]{\cite{tsanas2012accurate}}
	\footnotetext[4]{\cite{bock2004methods}}
	\footnotetext[5]{\cite{smith1988using}}
	\footnotetext[6]{\cite{angwin2016machine}}
	\footnotetext[7]{\cite{thompson2013video}}
	\footnotetext[8]{\cite{yeh2009comparisons}}
	\footnotetext[9]{\cite{cortez2009modeling}}
	
	Health insurance companies are considering the use of patterns of physical activity as measured by activity trackers in order to adjust health insurance rates of specific individuals \cite{sallis1998environmental,paluch2017leveraging}.  In fact, a recent clustering analysis found that different patterns of physical activity are correlated with different health outcomes \cite{fukuoka2018objectively}.  The objective of a health insurer in clustering activity data would be to find qualitative trends in an individual's physical activity that help categorize the risks that that customer portends. That is, individuals within these activity clusters are likely to incur similar levels of medical costs, and so it would be beneficial to engineer easy-to-spot features that can help insurers bucket customers. However, health insurance rates must satisfy a number of legal fairness considerations with respect to gender, race, and age.  This means that an insurance company may be found legally liable if the patterns used to adjust rates result in an unreasonably-negative impact on individuals of a specific gender, race, or age. Thus, an insurer may be interested in a feature engineering method to bucket customers while minimizing discrimination on protected attributes. Motivated by this, we use FPCA to perform a fair clustering of physical activity. \textit{Our goal is to find discernible qualitative trends in activity which are indicative of an individual's activity patterns, and thus health risks, but fair with respect to age.}
	
	We use minute-level data from the the National Health and Nutrition Examination Survey (NHANES) from 2005--2006 \cite{nhanes}, on the intensity levels of the physical activity of about 6000 women, measured over a week via an accelerometer.  In this example, we consider age to be our protected variable, specifically whether an individual is above or below 40 years of age.  We exclude weekends from our analysis, and average, over weekdays, the activity data by individual into 20-minute buckets. Thus, for each participant, we have data describing her average activity throughout an average day. We exclude individuals under 12 years of age, and those who display more than 16 hours of zero activity after averaging. The top 1\% most active participants, and corrupted data, were also excluded. Finally, data points corrupted or inexact due to accelerometer malfunctioning were excluded. This preprocessing mirrors that of \citeauthor{fukuoka2018objectively} and reflects practical concerns of insurers as well as the patchiness of accelerometer data.
	
	\begin{figure}
		\begin{center}
			\centerline{\includegraphics[width=\columnwidth]{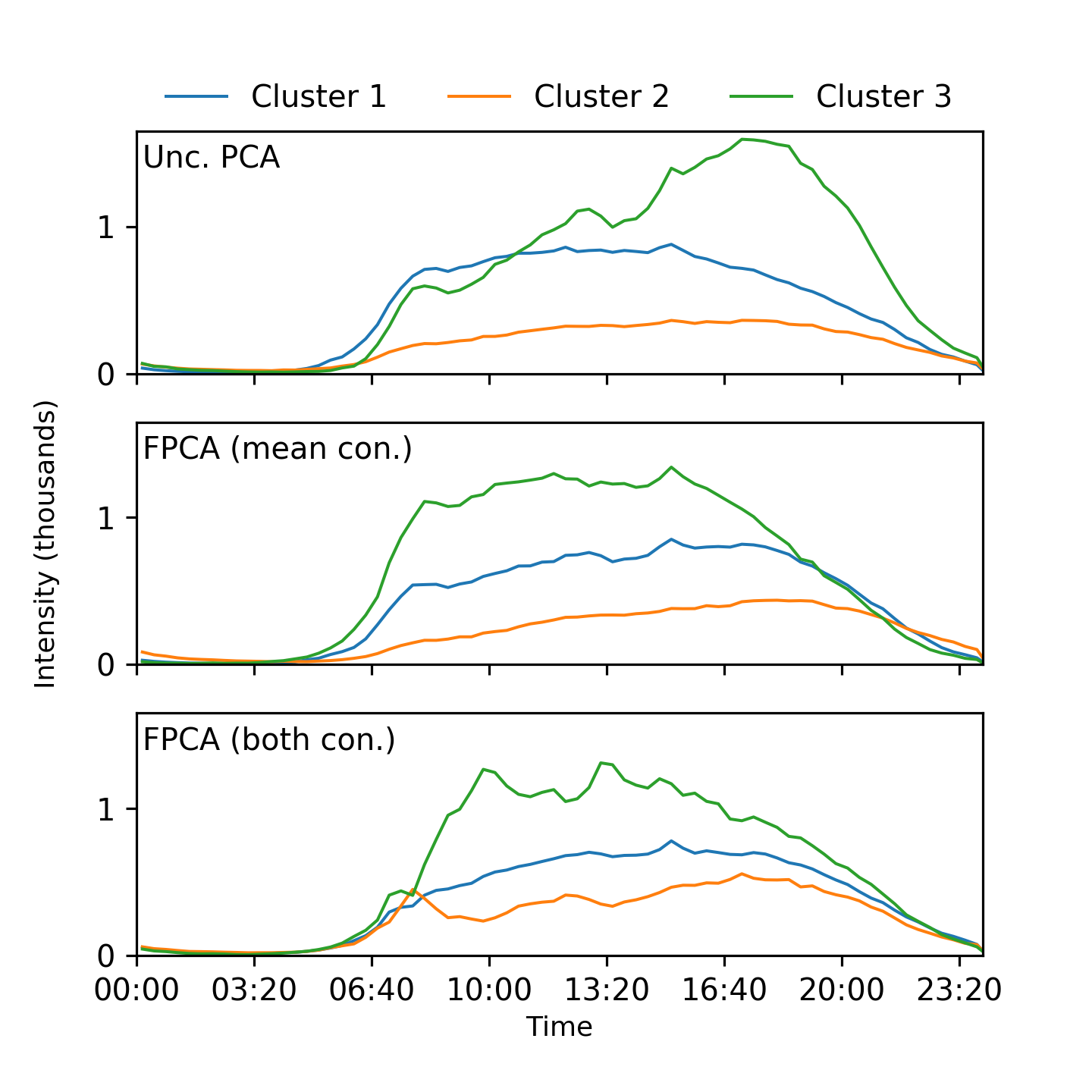}}
			\caption{The mean physical activity intensities, plotted throughout a day, of the clusters generated after dimensionality reduction through PCA, FPCA with the mean constraint, and FPCA with both constraints. In each plot, each line represents the average activity level of the members of one cluster.}
			\label{fig:nhanesAll}
		\end{center}
	\end{figure}
	
	PCA is sometimes used as a preprocessing step prior to clustering in order to expedite runtime. In this spirit, we find the top five principal components through PCA, FPCA with mean constraint, and FPCA with both constraints, with $\delta=0$ and $\mu=0.1$ throughout.  Then we conduct $k$-means clustering (with $k=3$) on the dimensionality-reduced data for each case. \Cref{fig:nhanesAll} displays the averaged physical activity patterns for the each of the clusters in each of the cases. Furthermore, \Cref{tab:clusterfairness} documents the proportion of each cluster comprised of examinees over 40. We note that the clusters found under an unconstrained PCA are most distinguishable after 3:00 PM, so an insurer interested in profiling an individual's risk would largely consider their activity in the evenings. However, we may observe in \Cref{tab:clusterfairness} that this approach results in notable age discrimination between buckets, opening the insurer to the risk of illegal price discrimination. On the hand, the second and third plots in \Cref{fig:nhanesAll} and columns in \Cref{tab:clusterfairness} suggest that clustering customers based on their activity during the workday, between 8:00 AM and 5:00 PM, would be less prone to discrimination.
	
	\begin{table}[t]
		\caption{The proportion of each cluster that are over 40 years of age. 36.05\% of all respondents are over 40. The final row displays the standard deviation of the numbers in the first three. The most fair solution would be the same age composition in all clusters, so this is a reasonable fairness metric.}
		\label{tab:clusterfairness}
		\begin{center}
			\begin{small}
				\begin{sc}
					\begin{tabular}{rcccr}
						\toprule
						& Unc. & Mean & Both \\
						\midrule
						\bf Cluster 1 & 43.18\% & 33.54\% & 35.61\% \\
						\bf Cluster 2 & 32.94\% & 38.64\% & 36.11\% \\
						\bf Cluster 3 & 8.71\% & 33.32\% & 37.28\% \\
						\midrule
						\bf Std. Dev & 14.87\% & 2.46\% & 1.79\% \\
						\bottomrule
					\end{tabular}
				\end{sc}
			\end{small}
		\end{center}
	\end{table}
	

	\section{Conclusion}
	\label{sec:conclusion}

	In this paper, we proposed a quantitative definition of fairness for dimensionality reduction, developed convex SDP formulations for fair PCA, and then demonstrated its effectiveness using several datasets.  Many avenues remain for future research on fair unsupervised learning.  For instance, we believe that our formulations in this paper may have suitable modifications that can be used to develop deflation and regression approaches for fair PCA analogous to those for sparse PCA \cite{AEJL:07,zou2006sparse}.

	\bibliography{fpca}
	\bibliographystyle{aaai}
	
	\appendix
	
	\section*{Appendix}
	
	\begin{table*}[!t]
		\caption{$\Delta$-fairness levels for the multivariate KS distance, for PCA, FPCA. and the method of \citeauthor{calmon2017optimized}. Best results for each fairness metric are bolded.}
		\label{tab:resultsKS}
		\begin{center}
			\begin{small}
				\begin{sc}
					\begin{tabular}{lccccr}
						\toprule
						Data Set & Unconstrained & FPCA - Mean & FPCA - Both & Calmon et al. \\
						\midrule
						Adult Income & 0.25 & 0.16 & \bf 0.07 & 0.25 \\
						Biodeg & 0.16 & \bf 0.15 & 0.17 & \bf 0.15 \\
						Ecoli & 0.64 & 0.29 & 0.32 & \bf 0.25 \\
						Energy & 0.16 & 0.12 & \bf 0.1 & 0.18\\
						German Credit & 0.17 & 0.16 & 0.16 & \bf 0.13 \\
						Image Seg & 0.19 & \bf 0.16 & 0.17 & 0.21 \\
						Letter Rec & 0.57 & \bf 0.09 & \bf 0.09 & 0.24 \\
						Magic & 0.14 & \bf 0.09 & 0.12 & 0.16 \\
						Pima Diabetes & 0.33 & 0.19 & \bf 0.18 & \bf 0.18 \\
						Recidivism & 0.20 & 0.09 & \bf 0.07 & 0.08\\
						SkillCraft & 0.12 & \bf 0.08 & \bf 0.08 & \bf 0.08 \\
						Statlog & 0.45 & 0.17 & \bf 0.12 & 0.18\\
						Steel & 0.48 & \bf 0.10 & \bf 0.10 & 0.58 \\
						Taiwanese Credit & 0.12 & \bf 0.07 & 0.08 & 0.13\\
						Wine Quality & 0.58 & 0.20 & \bf 0.07 & 0.44 \\
						\bottomrule
					\end{tabular}
				\end{sc}
			\end{small}
		\end{center}
	\end{table*}
	
	\subsection*{Parameters for FPCA}
	
	Here we present some additional experimental results. All results presented in this section are after averaged over 5 rounds of 70-30 training-testing splits, where an approach was trained on a random 70\% of the data and evaluated based on the specified metrics using the remaining 30\% of the data. In each case, the data was dimensionality-reduced using the top 5 principal components, fair or otherwise. All results follow after normalizing data columns, a practice that is common for datasets in which different features are of incomparable units. All results here use $\delta=0,\mu=0.01$.
	
	\begin{table*}[!b]
		\caption{Average squared distance from cluster center, as well as standard deviation of the proportion of each cluster that is of a certain protected class, for PCA, FPCA and the method of \citeauthor{calmon2017optimized}. Best fairness results for each dataset are bolded.}
		\label{tab:resultsCluster}
		\begin{center}
			\begin{small}
				\begin{sc}
					\begin{tabular}{l|cc|cc|cc|ccr}
						\toprule
						& \multicolumn{2}{c}{Unconstrained} & \multicolumn{2}{c}{FPCA - Mean} & \multicolumn{2}{c}{FPCA - Both} & \multicolumn{2}{c}{Calmon et al.} & \\
						Data Set & Score & Std. Dev & Score & Std. Dev & Score & Std. Dev & Score & Std. Dev \\
						\midrule
						Adult Income & 0.19 & 12.43 & 0.23 & 7.57 & 0.29 & \bf 2.28 & 0.05 & 11.32 \\
						Biodeg & 0.27 & 6.87 & 0.27 & 6.16 & 0.27 & \bf 5.34 & 0.16 & 5.49 \\
						Ecoli & 0.08 & 19.66 & 0.05 & 12.2 & 0.09 & \bf 10.69 & 0.18 & 11.78 \\
						Energy & 0.08 & 3.99 & 0.13 & 3.75 & 0.13 & \bf 3.57 & 0.10 & 5.02 \\
						German Credit & 0.25 & 6.4 & 0.25 & 4.82 & 0.28 & \bf 3.88 & 0.03 & 4.16 \\
						Image Seg & 0.10 & 8.46 & 0.09 & \bf 4.82 & 0.11 & 5.95 & 0.12 & 10.85 \\
						Letter Rec & 0.27 & 16.33 & 0.25 & 3.38 & 0.23 & \bf 3.28 & 0.37 & 8.65 \\
						Magic & 0.20 & 9.26 & 0.31 & \bf 5.15 & 0.35 & 5.42 & 0.18 & 8.77 \\
						Pima Diabetes & 0.24 & 9.09 & 0.27 & 6.36 & 0.26 & 5.96 & 0.28 & \bf 5.72 \\
						Recidivism & 0.26 & 7.6 & 0.17 & \bf 3.7 & 0.19 & 3.8 & 0.05 & 4.69 \\
						SkillCraft & 0.21 & 4.57 & 0.21 & \bf 2.27 & 0.24 & 2.88 & 0.38 & 3.21 \\
						Statlog & 0.09 & 21.99 & 0.23 & 16.06 & 0.31 & \bf 10.18 & 0.13 & 11.12 \\
						Steel & 0.16 & 18.49 & 0.19 & 9.85 & 0.24 & \bf 4.22 & 0.22 & 17.97 \\
						Taiwanese Credit & 0.17 & 3.85 & 0.24 & 2.99 & 0.29 & \bf 2.67 & 0.03 & 3.64 \\
						Wine Quality & 0.22 & 22.41 & 0.29 & 11.77 & 0.35 & \bf 2.11 & 0.34 & 11.70 \\
						\bottomrule
					\end{tabular}
				\end{sc}
			\end{small}
		\end{center}
	\end{table*}
	
	\subsection*{Benchmarks}
	
	To the best of our knowledge, there are very few methods that are directly comparable to ours. Most existing work is married to an explicit classification task, while ours is a general pre-processing step that makes it amenable to any type of analysis. Among the few comparable approaches are those of \citeauthor{zemel2013learning} \shortcite{zemel2013learning} and \citeauthor{calmon2017optimized}. Both design non-parametric optimization problems that yield a conditional distribution, $f_{\hat{X},\hat{Y}|X,Y,Z}$, which can then be used to transform data in a probabilistic way. We compare our method to that of \citeauthor{calmon2017optimized}, as their formulation is an extension of that of \citeauthor{zemel2013learning}.
	
	This method minimizes some pre-defined notion of overall deviation of $f_{\hat{X},\hat{Y}}$ from $f_{X,Y}$. In the original work, the authors choose to minimize $\frac{1}{2}\sum_{x,y}\left|f_{\hat{X},\hat{Y}}(x,y)-f_{X,Y}(x,y)\right|$. They subjects this to constraints on point-wise distortion $E_{\hat{X},\hat{Y}|X,Y}[\delta((X,Y),(\hat{X},\hat{Y})]$ for some function $\delta:\left\lbrace\mathbb{R}^p\times\{\pm1\}\right\rbrace^2\rightarrow\mathbb{R}_+$. It also bounds the dependency of the new main label $\hat{Y}$ on the original protected label, $J\left(f_{\hat{Y}|Z}(y|z),f_{Y}(y)\right)$, where they define $J$ to be the probability ratio measure:
	
	\begin{equation*}
	\label{eq:probratio}
	J(a,b) = \left|\frac{a}{b}-1\right|.
	\end{equation*}
	
	Thus, the final formulation is as follows:
	
	\begin{eqnarray*}\label{eq:calmon}
		\min&&\frac{1}{2}\sum_{x,y}\left|f_{\hat{X},\hat{Y}}(x,y)-f_{X,Y}(x,y)\right|\\
		\textrm{s.t.}&&E_{\hat{X},\hat{Y}|X,Y}[\delta((X,Y),(\hat{X},\hat{Y})|x,y]\le c,\forall x,y\label{eq:calmon1}\\
		&&\left|\frac{1}{f_{Y}(y)}f_{\hat{Y}|Z}(y|z)-1\right|\le d, \forall y,z\label{eq:calmon2}\\
		&&f_{\hat{X},\hat{Y}|X,Y,Z}\textrm{ are all distributions.}
	\end{eqnarray*}
	
	Following the authors, we approximate $f_{X,Y,Z}$ with the empirical distribution of the original data, separated into a pre-selected number of bins. Note that the resulting optimization problem will have $8(\#\textrm{bins})^{2p}$ parameters, and so can become computationally infeasible for high-dimensional datasets. To account for this, we follow the example of the original work and choose the 3 features most correlated with the main label, $y$. Each dimension is split into 8 bins. We choose $\delta((x',y'),(x,y))$ to be $0$ if $y=y'$ and $x=x'$, $0.5$ if $y=y'$ and $x,x$ vary by at most one in any dimension, and $1$ otherwise, which is similar to the $\delta$ chosen by the authors themselves. Finally, $c$ and $d$ were set at 0.1 and 0.3, respectively.
	
	\subsection*{Experiments}
	
	In \cref{tab:resultsKS}, we present fairness results using the family $\mathcal{F}_c$ of multivariate CDF's described in Section 3.3 of the main document (analogous to Kolmogorov-Smirnov distance). We run this for unconstrained PCA, FPCA with only the mean constraint, FPCA with both constraints, and the method of \citeauthor{calmon2017optimized}. We observe that our methods greatly improve fairness by this metric as well.
	
	\begin{table*}[!t]
		\caption{Comparison of accuracy and fairness on classification task using linear SVM. Results shown for linear SVM after dimensionality-reduction via PCA, FPCA with just the mean constraint and FPCA with both constraints, and are compared to the FSVM method of \citeauthor{olfat2017spectral} \shortcite{olfat2017spectral} (run with $\delta=0, \mu=0.1$ on non-dimensionality-reduced data) and the non-parametric method of \citeauthor{calmon2017optimized}. Best fairness results are bolded.}
		\label{tab:resultsSVM}
		\begin{center}
			\begin{small}
				\begin{sc}
					\begin{tabular}{l|cc|cc|cc|cc|ccr}
						\toprule
						& \multicolumn{2}{c}{FSVM (no PCA)} &
						\multicolumn{2}{c}{Unconstrained} & \multicolumn{2}{c}{FPCA - Mean} & \multicolumn{2}{c}{FPCA - Both} & \multicolumn{2}{c}{Calmon et al.} & \\
						Data Set & AUC & $\Delta$ & AUC & $\Delta$ & AUC & $\Delta$ & AUC & $\Delta$ & AUC & $\Delta$ \\
						\midrule
						Adult Income & 0.86 & 0.13 & 0.66 & 0.17 & 0.69 & \bf 0.07 & 0.57 & 0.08 & 0.51 & 0.23 \\
						Biodeg & 0.85 & 0.12 & 0.82 & 0.20 & 0.81 & 0.13 & 0.79 & \bf 0.11 & 0.60 & 0.14 \\
						Ecoli & 0.74 & \bf 0.17 & 0.84 & 0.50 & 0.69 & 0.23 & 0.72 & 0.29 & 0.63 & 0.30 \\
						Energy & 0.55 & 0.09 & 0.51 & 0.09 & 0.56 & 0.08 & 0.55 & \bf 0.07 & 0.54 & 0.13 \\
						German Credit & 0.76 & 0.11 & 0.62 & 0.11 & 0.57 & \bf 0.10 & 0.58 & 0.14 & 0.63 & 0.11 \\
						Image Seg & 0.99 & 0.19 & 0.99 & 0.16 & 0.99 & 0.19 & 0.98 & \bf 0.15 & 0.79 & 0.20 \\
						Letter Rec & 0.72 & \bf 0.07 & 0.58 & 0.60 & 0.50 & 0.09 & 0.49 & 0.10 & 0.65 & 0.19 \\
						Magic & 0.83 & 0.13 & 0.74 & 0.14 & 0.82 & 0.13 & 0.72 & \bf 0.12 & 0.65 & 0.13 \\
						Pima Diabetes & 0.80 & 0.14 & 0.75 & 0.21 & 0.73 & \bf 0.11 & 0.76 & 0.15 & 0.54 & 0.15 \\
						Recidivism & 0.54 & 0.08 & 0.69 & 0.24 & 0.54 & \bf 0.06 & 0.52 & 0.07 & 0.55 & 0.08 \\
						SkillCraft & 0.82 & 0.06 & 0.85 & 0.10 & 0.82 & \bf 0.05 & 0.80 & \bf 0.05 & 0.62 & 0.07 \\
						Statlog & 0.99 & 0.31 & 1.00 & 0.33 & 0.99 & 0.33 & 0.85 & 0.18 & 0.67 & \bf 0.16 \\
						Steel & 0.73 & 0.15 & 0.53 & 0.37 & 0.62 & 0.19 & 0.61 & \bf 0.12 & 0.55 & 0.15 \\
						Taiwanese Credit & 0.73 & \bf 0.07 & 0.60 & 0.11 & 0.60 & 0.09 & 0.64 & \bf 0.07 & 0.75 & \bf 0.07 \\
						Wine Quality & 0.78 & 0.10 & 0.69 & 0.75 & 0.69 & 0.19 & 0.67 & \bf 0.05 & 0.66 & 0.09 \\
						\bottomrule
					\end{tabular}
				\end{sc}
			\end{small}
		\end{center}
	\end{table*}
	
	In \cref{tab:resultsCluster}, we present statistics for clustering done transformed data. Again, the methods used to transofrm the data are PCA, FPCA with only the mean constraint, FPCA with both constraints, and the method of \citeauthor{calmon2017optimized}. Reducing dimensionality prior to clustering is a common technique used to avoid the curse-of-dimensionality that arises in many unsupervised methods \cite{kumar2010clustering,aggarwal2001surprising}, so this is a relevant metric of comparison. For each case, we display the average squared distance from the closest cluster as a measure of accuracy, and the standard deviation of the proportion of each cluster that is of a certain protected class (the same metric reported in Section of 6.4 of the main document). That is, we consider the proportion of each cluster that is of protected class $z=+1$ (in percentage points), and return the standard deviation of these figures (so the units would also be percentage points for these columns). In a given clustering, it is intuitive that the most fair outcome would be for every cluster to have the same composition in terms of protected classes (thus standard deviation of zero as mentioned above), so we maintain that this is a reasonable proxy for fairness. We observe that our method greatly reduces the unfairness within clusters, while not significantly decreasing the value of the clustering compared to a typical clustering. In some cases, we note that our method does even better in terms of accuracy; this may arise due to the fact that we are evaluating based on testing error as opposed to training error (i.e. we find cluster centers on training data and then find the closest cluster center for each point in the testing set). This suggests that our method may even act to aid in reducing generalization error.
	
	Finally, we present an analysis of our method as a preprocessing step for classification in \cref{tab:resultsSVM}. Here, we define a classification task on the datasets, and show the performance of linear SVM after dimensionality reduction via PCA, FPCA with the mean constraint and FPCA with both constraints. We compare these  all with the method of \citeauthor{calmon2017optimized}, as before, but we also compare to the Fair SVM (FSVM) method of \cite{olfat2017spectral} (run with hyperparameters $\delta=0,\mu=0.1$ on non-dimensionality-reduced data), which was specifically designed for such a task. We compared the datasets based on fairness, as well as Area Under the Curve (AUC), which is measured as the area under the ROC curve of a classifier that takes a threshold as an input. We note that our method often produces more fair results. In some cases, our method matches or even beats the accuracy of FSVM. It is of importance that our method is a flexible method, while FSVM is specifically tailored to margin classifiers. Thus, it is to be expected that our method would not be strictly better in terms of accuracy. However, the comparison with regards to fairness is often quite favorable for our method.
	
\end{document}